\documentclass[runningheads]{llncs}
\usepackage[T1]{fontenc}
\usepackage{graphicx,verbatim}
\usepackage{mathtools}
\usepackage{amssymb}
\usepackage{tabularx}
\usepackage{multirow}
\usepackage{subfigure}
\usepackage{bm}

\begin{document}
\title{PCaPaint: Prostate Cancer Inpainting by Mitigating Shortcut Learning}
%

\author{Levente Lippenszky \and
Hongxu Yang \and 
Marcell Dömötör \and 
Krisztian Koos \and 
László Ruskó}  
\authorrunning{L. Lippenszky et al.}
\institute{GE HealthCare \\
    \email{levente.lippenszky@gehealthcare.com}}
  
\maketitle              
\begin{abstract}
The development of AI systems for tumor-specific applications is limited by the scarcity of labeled data. Synthetic tumor inpainting offers a promising approach but faces challenges for prostate cancer MRI which contains high-resolution multi-sequence data. Although methods leveraging latent diffusion models (LDMs) enable large-volume synthesis, they are prone to shortcut learning, simply reproducing the condition image created by masking the lesion region. In this work, we introduce PCaPaint, a prostate cancer inpainting method based on LDMs that explicitly addresses this failure mode. To overcome shortcut learning that compromises synthetic tumor texture, we propose a simple yet efficient conditioning strategy in which the condition image is filled with Gaussian noise, and we provide theoretical justification. In addition, we propose a novel training objective for LDM 
that emphasizes the error within the lesion region. Furthermore, we introduce a multi-sequence latent design, in which T2w scans and DWI\&ADC scans are compressed using two separate autoencoders  to preserve their distinct frequency characteristics. Extensive experiments demonstrate that the generated synthetic data improves downstream performance in prostate lesion segmentation, patient-level classification and lesion-level detection. Furthermore, our method significantly outperforms a recent state-of-the-art LDM-based tumor inpainting method both in downstream performance and in synthetic image quality.


\end{abstract}
\section{Introduction}
The development of AI systems in healthcare is often constrained by the scarcity of labeled data~\cite{grabke2025mitigating,jin2023label}. This challenge is even more pronounced in tumor‑specific applications, where pathological cases are far less common than normal patient data~\cite{zhang2024lefusion}. Furthermore, developing robust deep learning models for these applications requires substantial amount of data with sufficient diversity~\cite{liu2023clip,wang2021annotation,chou2024acquiring}. However, collecting these datasets is expensive, and the difficulty of accurately annotating tumors intensifies this challenge~\cite{chen2024towards,yang2025synbt}.

Synthetic tumor inpainting offers a promising approach for generating diverse lesion data while eliminating the need for manual annotation. However, prostate biparametric MRI (bpMRI) synthesis poses significant challenges as it involves high-resolution 3D data and three MRI sequences with substantially different frequency characteristics. In recent years, tumor inpainting methods based on diffusion models~\cite{ho2020denoising} have emerged covering a large variety of lesion types. LeFusion~\cite{zhang2024lefusion} incorporates forward-diffused background context into the reverse diffusion process, but its image-space design restricts its use for prostate bpMRI. Other state-of-the-art (SOTA) lesion inpainting methods, such as DiffTumor~\cite{chen2024towards} and SynBT~\cite{yang2025synbt}, leverage latent diffusion models (LDMs)~\cite{rombach2022high}. These approaches use a healthy image as condition where the lesion region is zeroed out, and the LDM is trained to reconstruct the lesion. While latent representations enable prostate bpMRI synthesis, this conditioning strategy can lead to shortcut learning, causing the model to reconstruct the healthy image rather than inpaint lesion texture. Fig.~\ref{fig:shortcut_learning} illustrates shortcut learning in DiffTumor~\cite{chen2024towards}. Fig.~\ref{fig:shortcut_learning} (a) shows a ground truth T2‑weighted (T2w) training example with a real lesion, and (b) the corresponding healthy condition image with the lesion region filled with zeros (gray). After 200,000 training steps, the generated output in (c) exhibits the shortcut failure mode, reproducing the healthy image instead of synthesizing lesion texture. Fig.~\ref{fig:shortcut_learning} (d) and (e) show the same behavior when the conditioning image is filled with $-1$ values (black).

In this work, we propose PCaPaint, a novel prostate cancer inpainting method based on LDMs that generates realistic synthetic lesions in healthy prostate bpMRI cases, while explicitly addressing shortcut learning. Our main \textbf{contributions} are as follows.
\begin{enumerate}
    \item 
    We introduce a simple yet efficient conditioning strategy where the condition image is filled with Gaussian noise and provide theoretical justification that it mitigates shortcut learning. In addition, we propose a training objective for LDM that emphasizes the error within the lesion region to further mitigate this failure mode.
    \item We introduce a novel multi‑sequence latent design, in which T2w scans and DWI\&ADC scans are compressed using two separate autoencoders to preserve their distinct frequency characteristics.
    \item Extensive experiments show that our synthetic data improves performance in three downstream tasks. Moreover, our method significantly outperforms a SOTA LDM-based tumor inpainting method both in downstream performance and in synthetic image quality.
\end{enumerate}

\begin{figure}[t]
\includegraphics[width=\textwidth]{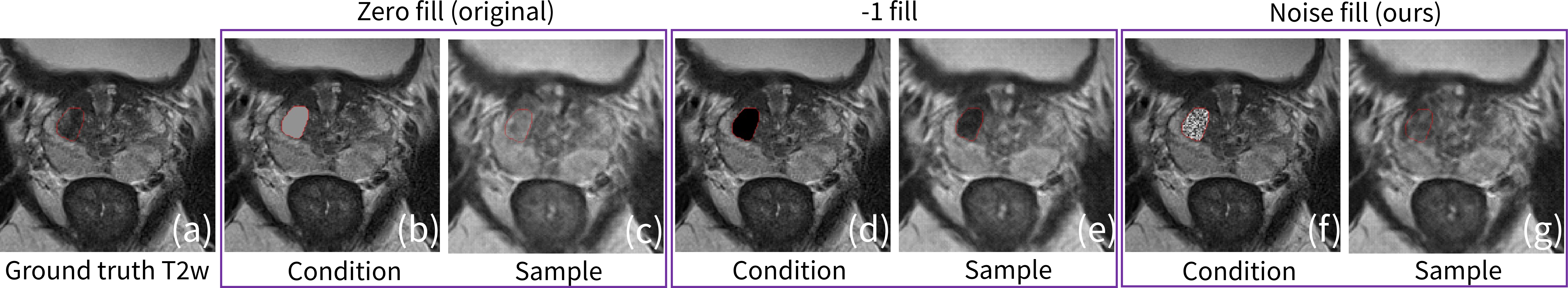}
\caption{DiffTumor~\cite{chen2024towards} samples under different conditioning strategies. Panels show the condition and sample for zero fill (gray),  $-1$ fill (black) and the proposed noise fill.}\label{fig:shortcut_learning}
\end{figure}

\section{Methods}
\subsection{Conditioning with Noise Fill} \label{subsec:noise_fill}
We provide theoretical justification that filling the lesion region in the condition image with noise mitigates shortcut learning. Proposition~\ref{prop:noise_fill} assumes a diffusion model~\cite{ho2020denoising}, but the reasoning extends to LDM~\cite{rombach2022high}. This is illustrated qualitatively in Fig.~\ref{fig:shortcut_learning} (f) and (g) where DiffTumor~\cite{chen2024towards} generates improved lesion texture, and quantitatively in Section~\ref{subsec:img_quality_eval}. 

Let $\bm{x}_0$ denote the lesion image and $\bm{h}$ the healthy condition image. The healthy image is defined as $\bm{h} = (1 - \bm{m}) \odot \bm{x}_0 + \bm{m} \odot \bm{\eta}$, where $\bm{m}$ is the lesion mask and $\bm{\eta}$ is the tensor used to fill the lesion region. In LDM-based tumor inpainting methods, $\bm{\eta}$ is the all-zero tensor~\cite{yang2025synbt,chen2024towards}. We define the lesion difference as $\bm{d} \coloneqq \bm{x}_0 - \bm{h} = \bm{m} \odot (\bm{x}_0 - \bm{\eta})$. The diffusion forward process defines the noisy lesion image at timestep $t$ as
$\bm{x}_t = \sqrt{\bar{\alpha}_t}\bm{x}_0 + \sqrt{1 - \bar{\alpha}_t}\bm{\epsilon}$, where $\bm{\epsilon} \sim \mathcal{N}(\bm{0}, \bm{I})$, $\bar{\alpha}_t = \prod_{s=1}^t \alpha_s$ and $\alpha_t = 1 - \beta_t$, with $\{\beta_t\}_{t=1}^T$ denoting the noise schedule and $T$ the number of time steps. The diffusion model $\bm{\epsilon}_\theta$ predicts the noise as $\hat{\bm{\epsilon}} = \bm{\epsilon}_\theta(\bm{x}_t, t, \bm{h}, \bm{m})$. Fig.~\ref{fig:shortcut_learning} (b)-(e) shows that DiffTumor learns a shortcut and denoises $\bm{x}_t$ by simply reconstructing the healthy image $\bm{h}$ when the fill tensor is constant. We assume the diffusion model exploits a shortcut by approximating the analytically computable residual $\bm{r}_t$ as its noise prediction
\begin{equation}
\bm{r}_t = \frac{\bm{x}_t - \sqrt{\bar{\alpha}_t} \, \bm{h}}{\sqrt{1 - \bar{\alpha}_t}}.    
\end{equation}
This solution depends only on the inputs $(\bm{x}_t, \bm{h}, \bm{m}, t)$ available to the network and the fixed noise schedule.

\begin{proposition} \label{prop:noise_fill}
Assume the network exploits the shortcut $\hat{\bm{\epsilon}}=\bm{r}_t$.
Consider (i) constant fill $\bm{\eta} = \bm{c}$ and
(ii) Gaussian noise fill $\bm{\eta} \sim \mathcal{N}(\bm{c}, \sigma^2 \bm{I})$, where $\bm{c}$ is a constant tensor and $\sigma > 0$ denotes the standard deviation. Then, the difference of the expected training losses satisfies
\begin{equation}
\mathcal{L}_{\mathrm{noise}} - \mathcal{L}_{\mathrm{const}} = \kappa \sigma^2 \mathbb{E}_{\bm{x}_0} \left[ n_m(\bm{x}_0)\right] > 0 \, ,
\end{equation}
where $t \sim \mathrm{Uniform}(\{1, ..., T\})$, $\kappa_t \coloneqq \frac{\sqrt{\bar{\alpha}_t}}{\sqrt{1-\bar{\alpha}_t}}$ and $\kappa = \mathbb{E}_{t} \left[ \kappa_t^2 \right] > 0$, $n_m(\bm{x}_0) > 0$ is the number of masked voxels for $\bm{x}_0$. Noise fill increases the loss for the shortcut solution, encouraging the network to learn the intended denoising.
\end{proposition}

\begin{proof}
Substituting $\bm{x}_t$ into $\bm{r}_t$ gives 
\begin{equation}
\bm{r}_t = \frac{\sqrt{\bar{\alpha}_t} \, \bm{x}_0 + \sqrt{1 - \bar{\alpha}_t} \, \bm{\epsilon} - \sqrt{\bar{\alpha}_t} \, \bm{h}}{\sqrt{1 - \bar{\alpha}_t}} = 
\bm{\epsilon} + \frac{\sqrt{\bar{\alpha}_t}\,(\bm{x}_0 - \bm{h})}{\sqrt{1 - \bar{\alpha}_t}} = 
\bm{\epsilon} + \kappa_t \bm{d} \, .
\end{equation}
The residual is the sum of the true noise and the scaled lesion difference. The shortcut behavior dominates as the lesion mask $\bm{m}$ approaches the empty mask, since $\bm{r}_t$ then converges to $\bm{\epsilon}$. The training loss $\mathcal{L}$ takes the following form
\begin{equation}
\mathbb{E} \left[ \| \hat{\bm{\epsilon}} - \bm{\epsilon} \|^2 \right] = 
\mathbb{E}_{t, \bm{x}_0, \bm{m}, \bm{\epsilon}, \bm{\eta}} \left[ \| \kappa_t \bm{d} \|^2 \right] = 
\mathbb{E}_{t} \left[ \kappa_t^2 \right] \mathbb{E}_{\bm{x}_0, \bm{m}, \bm{\eta}} \left[ \| \bm{m} \odot (\bm{x}_0 - \bm{\eta}) \|^2 \right] \, . 
\end{equation}
We define the diagonal matrix $\bm{M} \coloneqq \operatorname{diag}(\operatorname{vec}(\bm{m}))$ that performs the same masking operation as $\bm{m}$ when the images are treated as vectors. We assume the lesion mask is provided for each $\bm{x}_0$ via manual annotation, and we treat $\bm{M}=\bm{M}(\bm{x}_0)$ as deterministic given $\bm{x}_0$. Using the law of total expectation and $\bm{\eta}\perp \bm{x}_0$,
\begin{equation} \label{eq5}
\mathcal{L} = \kappa \mathbb{E}_{\bm{x}_0} \left[ \mathbb{E}_{\bm{\eta}} \left[ \| \bm{M}(\bm{x}_0) (\bm{x}_0 - \bm{\eta}) \|^2 \mid \bm{x}_0 \right] \right] \, .
\end{equation}
By the bias-variance decomposition~\cite{bishop2006pattern}, the expected squared $L_2$ norm of a random vector $\bm{X}$ can be decomposed as $\mathbb{E} \left[ \|\bm{X}\|^2 \right] = \operatorname{tr}(\operatorname{Cov}(\bm{X})) + \| \mathbb{E}[\bm{X}] \|^2$. To decompose the inner expectation in Eq.~\eqref{eq5}, we define $\bm{Z} \coloneqq \bm{M}(\bm{x}_0) (\bm{x}_0 - \bm{\eta})$. Then $\mathbb{E} \left[ \bm{Z} \mid \bm{x}_0 \right] = \bm{M}(\bm{x}_0) (\bm{x}_0 - \bm{\mu}_{\eta})$ and $\operatorname{Cov} ( \bm{Z} \mid \bm{x}_0 ) = \bm{M}(\bm{x}_0) \bm{\Sigma}_{\eta} \bm{M} (\bm{x}_0)^T$ where $\bm{\mu}_{\eta} \coloneqq \mathbb{E}[\bm{\eta}]$ and $\bm{\Sigma}_{\eta} \coloneqq \operatorname{Cov}(\bm{\eta})$. Applying the decomposition yields 
\begin{equation} \label{eq6}
\mathbb{E} \left[ \| \bm{Z} \|^2 \mid \bm{x}_0 \right] = \operatorname{tr} \left(\operatorname{Cov} \left(\bm{Z} \mid \bm{x}_0 \right) \right) + \left\| \mathbb{E} \left[ \bm{Z} \mid \bm{x}_0 \right] \right\| ^2 \, . 
\end{equation} 
Substituting into Eq.~\eqref{eq5} gives
\begin{equation} \label{eq7}
\mathcal{L} = \kappa \mathbb{E}_{\bm{x}_0} \left[ \operatorname{tr} \left( \bm{M}(\bm{x}_0) \bm{\Sigma}_{\eta} \bm{M} (\bm{x}_0) ^T \right) + \left\| \bm{M}(\bm{x}_0) (\bm{x}_0 - \bm{\mu}_{\eta}) \right\| ^2
\right] \, .
\end{equation}
In (i) constant fill, $\bm{\Sigma}_{\bm{\eta}} = \bm{0}$ and $\bm{\mu}_{\bm{\eta}} = \bm{c}$, which implies that
\begin{equation} \label{eq8}
\mathcal{L}_{\mathrm{const}} = \kappa \mathbb{E}_{\bm{x}_0} \left[ \left\| \bm{M} (\bm{x}_0) (\bm{x}_0 - \bm{c}) \right\| ^2 \right] \, .
\end{equation}
In (ii) noise fill, $\bm{\Sigma}_{\bm{\eta}} = \sigma^2 \bm{I}, \bm{\mu}_{\bm{\eta}} = \bm{c}$ and $\operatorname{tr} (\bm{M} (\bm{x}_0) \sigma^2 \bm{I} \bm{M} (\bm{x}_0)^T) = \sigma^2 n_m (\bm{x}_0)$, which yields  
\begin{equation} \label{eq9}
\mathcal{L}_{\mathrm{noise}} = \kappa \mathbb{E}_{\bm{x}_0} \left[ \sigma^2 n_m (\bm{x}_0) + \left\| \bm{M} (\bm{x}_0) (\bm{x}_0 - \bm{c}) \right\| ^2 \right] \, .
\end{equation}
Subtracting Eq.~\eqref{eq8} from Eq.~\eqref{eq9} completes the proof.
\end{proof}

\subsection{PCaPaint} \label{subsec:pcapaint}
\subsubsection{Autoencoders}
In the first stage, T2w scans and DWI\&ADC scans are compressed separately to reflect their distinct characteristics: T2w contains rich high-frequency anatomical detail, while ADC and DWI exhibit lower-frequency intensity patterns. We employ the vector-quantized variational autoencoder with adversarial training (VQGAN)~\cite{esser2021taming} for both autoencoders, $i \in \{1, 2\}$. We denote the T2w image as $\bm{x}^{(1)} \in \mathbb{R}^{H \times W \times D}$, where $H, W$ and $D$ represent the height, width, and depth, respectively. We further define $\bm{x}^{(2)} \in \mathbb{R}^{H \times W \times D \times 2}$, representing the DWI and ADC volumes concatenated along the channel dimension. The input MRI tensor $\bm{x}^{(i)}$ is mapped to a continuous latent representation $\bm{z}_e^{(i)} \in \mathbb{R}^{h \times w \times d \times c}$ via the encoder $f^{(i)}$, where $h, w, d$ denote the latent space dimensions and $c$ is the number of channels. In the vector quantization step, each $c$-dimensional latent vector in $\bm{z}_e^{(i)}$ is replaced by its nearest codebook entry from the learned codebook $\mathcal{C} = \{\bm{c}_k\}_{k=1}^K$, resulting in the quantized latent $\bm{z}_q^{(i)} = q(\bm{z}_e^{(i)}) \in \mathbb{R}^{h \times w \times d \times c}$, where $K$ is the codebook size. Lastly, the decoder $g^{(i)}$ reconstructs the quantized latent such that $\hat{\bm{x}}^{(i)} = g^{(i)}(\bm{z}_q^{(i)})$. The loss is defined as $\mathcal{L} = \mathcal{L}_{\text{recon}} + \mathcal{L}_{\text{commit}} + \lambda_{\text{percep}} \mathcal{L}_{\text{percep}} + \lambda_{\text{adv}} \mathcal{L}_{\text{adv}}$, where $\mathcal{L}_{\text{recon}} = \|\bm{x}^{(i)} - \hat{\bm{x}}^{(i)}\|_1$ and $\mathcal{L}_{\text{commit}} = \beta \| \text{sg}[\bm{z}_q^{(i)} - \bm{z}_e^{(i)}] \|_2^2$, with $\text{sg}[\cdot]$ denoting the stop-gradient operator and $\beta$ the commitment cost. The perceptual loss $\mathcal{L}_{\text{percep}}$ measures feature similarity between $\bm{x}^{(i)}$ and $\hat{\bm{x}}^{(i)}$ in a pretrained network, and $\mathcal{L}_{\text{adv}}$ computes the least-squares patch adversarial loss using a patch discriminator~\cite{isola2017image}.

\subsubsection{Latent Diffusion Model} 
In the second stage, we train an LDM~\cite{rombach2022high} conditioned on the healthy image and the lesion mask. Lesion bpMRI $\bm{x}_0$ is encoded to the latent space using the VQGAN encoders, $\bm{z}_0 = \text{concat}(\bm{z}_0^{(1)}, \bm{z}_0^{(2)})$, where $\bm{z}_0^{(i)} = f^{(i)}(\bm{x}_0^{(i)})$ and $i \in \{1, 2\}$. In the forward process, $\bm{z}_0$ is gradually transformed into standard Gaussian noise $\bm{z}_T \sim \mathcal{N}(\bm{0}, \bm{I})$ following a Markov process. The reverse process is parametrized by a denoising U-Net $\bm{\epsilon}_\theta$ that predicts the noise added at each time step. The healthy image is constructed by filling the lesion region with Gaussian noise. Formally, we generate standard Gaussian noise $\bm{\xi} \sim \mathcal{N}(\bm{0}, \bm{I})$, scale it by $\sigma > 0$ and map it to the $(-1, 1)$ range using the $\tanh$ function, $\bm{\eta} = \tanh(\sigma\bm{\xi})$. Then, $\bm{h}$ is defined as $\bm{h} = (1-\bm{m}) \odot \bm{x}_0 + \bm{m} \odot \bm{\eta}$ and encoded into latent $\bm{z}_h$ where $\bm{z}_h = \text{concat}(\bm{z}_h^{(1)}, \bm{z}_h^{(2)})$, $\bm{z}_h^{(i)} = f^{(i)}(\bm{h}^{(i)})$ and $i \in \{1, 2\}$. Furthermore, we downsample the lesion mask to match the latent space dimensions, $\tilde{\bm{m}} = \text{down}(\bm{m})$. The denoising U-Net receives $\bm{z}_h$ and $\tilde{\bm{m}}$ as conditions concatenated to $\bm{z}_t$ along the channel dimension. To further mitigate shortcut learning, we employ a lesion-weighted training objective that amplifies the error inside the lesion region:
\begin{equation} \label{eq10}
\mathbb{E}_{\bm{x}_0, \bm{m}, \bm{\epsilon}, \bm{\xi}, t} \left[ \| \bm{\epsilon} - \bm{\epsilon}_\theta(\bm{z}_t, t, \bm{z}_h, \tilde{\bm{m}}) \|^2 + \lambda \| \tilde{\bm{m}} \odot \bm{\epsilon} - \tilde{\bm{m}} \odot \bm{\epsilon}_\theta(\bm{z}_t, t, \bm{z}_h, \tilde{\bm{m}}) \|^2 \right] \, ,
\end{equation}
where $\lambda$ controls the contribution of the lesion-focused term.

\begin{figure}[t]
\centering
\includegraphics[width=0.83\textwidth]{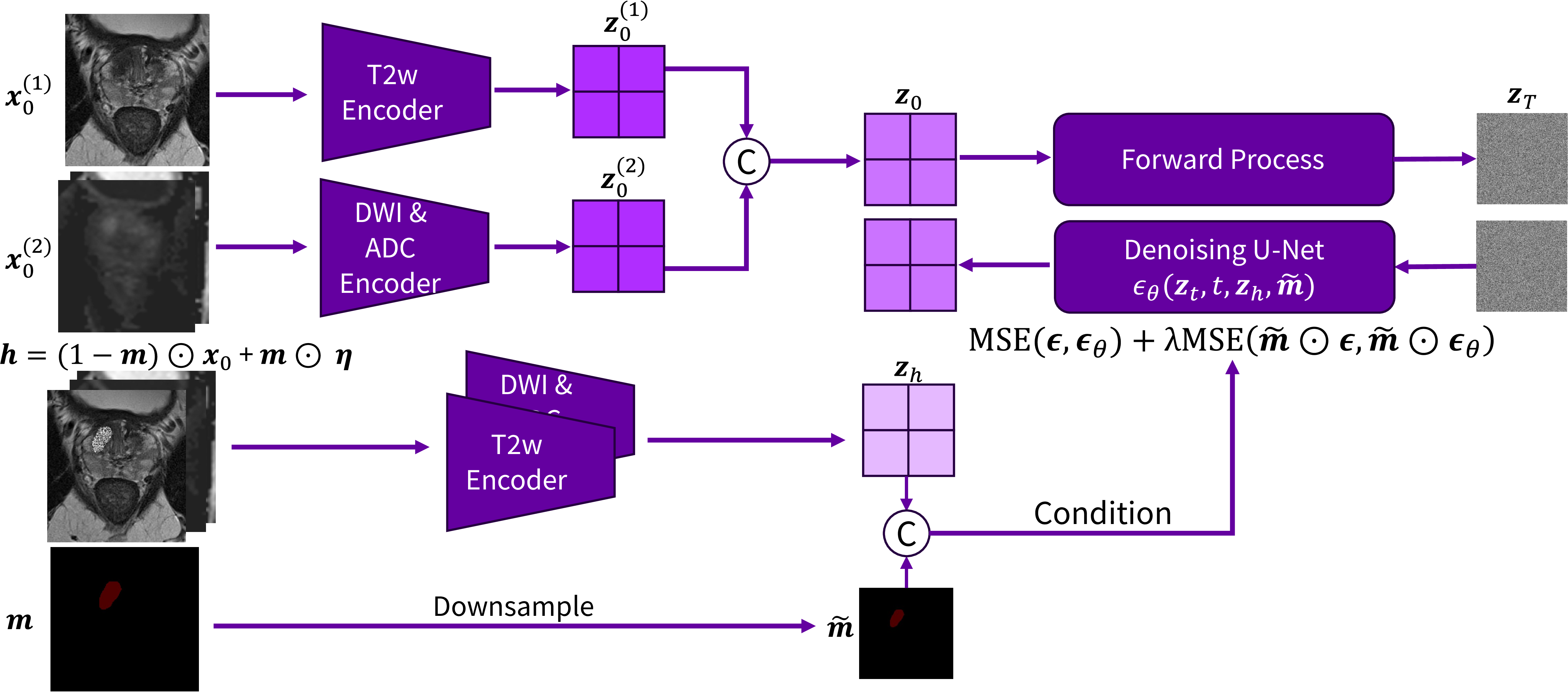}
\caption{Training the latent diffusion model (LDM). The lesion bpMRI $\bm{x}_0$ is encoded by the VQGAN encoders to form the latent $\bm{z}_0$. LDM is conditioned on the latent healthy image $\bm{z}_h$ and the downsampled lesion mask $\tilde{\bm{m}}$. The healthy image $\bm{h}$ is constructed by replacing the lesion region with Gaussian noise. LDM is trained with an objective that amplifies the error in the lesion region.}\label{fig:pcapaint}
\end{figure}

\subsubsection{Sampling}
In the third stage, we generate prostate lesion masks incorporating medical knowledge and inpaint synthetic lesions onto non-lesion bpMRIs. Lesion synthesis is performed by sampling the LDM using DDIM~\cite{song2020denoising}, conditioned on the healthy image and the lesion mask. Given peripheral zone (PZ) and transition zone (TZ) segmentation masks, we sample lesion center with probabilities of 0.75 from PZ and 0.25 from TZ~\cite{mcneal1988zonal}. We then sample a target lesion volume in voxels constrained to lie between minimum and maximum fractions of the whole gland volume. An ellipsoidal lesion is generated by drawing random axis ratios and scaling the ellipsoid to match the target volume while accounting for voxel spacing. To create realistic, irregular boundaries, we perturb the ellipsoid surface using Gaussian-filtered noise field scaled by the mean radius. Finally, we clip the mask to the gland by removing voxels outside the gland segmentation.

\section{Experimental Results}
\subsection{Experimental Setup}
\subsubsection{Data}
Training set consisted of 900 bpMRI cases from the publicly available PI-CAI~\cite{saha2022pi} dataset, including 252 cases with clinically significant prostate cancer (csPCa). The test set comprised of 599 bpMRI cases from PI-CAI, of which 172 were csPCa cases. Preprocessing was performed using the \texttt{picai\_prep} repository~\cite{saha2023artificial}. Each scan was resampled to a resolution of $0.5 \times 0.5 \times 3.0~\mathrm{mm}^3$ and cropped or padded to a spatial size of $256 \times 256 \times 32$ voxels. Image intensities were rescaled to the $[-1, 1]$ range.

\subsubsection{Implementation Details}
For PCaPaint, we trained the two VQGANs using the VQ-VAE~\cite{van2017neural} implementation of MONAI 1.5.0~\cite{cardoso2022monai} for 100,000 steps using the Adam optimizer. The learning rate was set to $5 \times 10^{-5}$ for the T2w VQGAN and to $10^{-4}$ for the DWI\&ADC VQGAN. We set the embedding dimensions to 64 and the codebook size to 8192 for both autoencoders. The denoising U-Net was trained using the DiffTumor~\cite{chen2024towards} implementation for 200,000 steps with Adam and a learning rate of $5 \times 10^{-5}$. The standard deviation of the Gaussian noise fill was set to $\sigma = 0.5$ such that most sampled values lie within the $[-1, 1]$ range. The weight of the second loss term in Eq.~\eqref{eq10} was empirically set to $\lambda = 150$. For DiffTumor, we trained a single VQGAN to compress the bpMRIs with a learning rate of $5 \times 10^{-5}$. To provide larger capacity for the autoencoder, we increased the embedding dimensions to 128. We applied the same training setup as for PCaPaint to ensure a fair comparison. Trainings were performed on an NVIDIA A100 Tensor Core GPU with 40 GB memory.

\subsection{Image Quality Evaluation} \label{subsec:img_quality_eval}
Synthetic image quality was evaluated for DiffTumor, DiffTumor with noise fill (NF) and PCaPaint by inpainting the lesion region of the 172 real bpMRIs in the test set. We computed structural similarity index measure (SSIM) and peak signal-to-noise ratio (PSNR) between the ground truth and synthetic images. In addition, we calculated PSNR separately for the foreground and background regions. Table~\ref{tab1} demonstrates that PCaPaint generates statistically significantly higher quality images compared to DiffTumor across MRI sequences and image regions. Foreground (fg) results show that DiffTumor with NF outperforms DiffTumor, while PCaPaint achieves the best image quality, highlighting the benefits of our noise fill conditioning and lesion-focused loss. PCaPaint also achieves the best performance in background (bg) metrics, underscoring the advantage of compressing bpMRIs using two separate autoencoders.

\begin{table}[t]
\centering
\caption{Synthetic image quality evaluation for DiffTumor~\cite{chen2024towards}, DiffTumor with noise fill (NF) and PCaPaint. SSIM and PSNR were computed between ground truth and synthetic images for test lesion bpMRIs for the entire image, foreground (fg) and background (bg). 95\% bootstrap CI widths are in brackets. Best and second-best are in \textbf{bold} and \underline{underline}, respectively. $^{*}$ indicates statistically significant differences between DiffTumor and PCaPaint at 5\% level using two-sided Wilcoxon signed-rank tests with Holm correction.}\label{tab1}
\begin{tabularx}{0.85\textwidth}{l|l|X|X|X}
\hline
MRI & Metric $\uparrow$ & DiffTumor & DiffTumor NF  & PCaPaint (Ours) \\
\hline
\multirow{4}{*}{T2w} 
  & SSIM    & 0.77 (0.01) & \underline{0.79} (0.01) & \textbf{0.83}$^{*}$ (0.02) \\
  & PSNR    & 20.75 (0.56) & \underline{21.58} (0.60) & \textbf{23.11}$^{*}$ (0.62) \\
  & PSNR fg & 12.38 (0.64) & \underline{16.26} (0.80) & \textbf{20.42}$^{*}$ (0.97) \\
  & PSNR bg & 20.80 (0.55) & \underline{21.60} (0.60) & \textbf{23.11}$^{*}$ (0.62) \\
\hline
\multirow{4}{*}{DWI} 
  & SSIM    & 0.51 (0.06) & 0.51 (0.06) & \textbf{0.64}$^{*}$ (0.05) \\
  & PSNR    & \underline{21.39} (0.66) & 21.16 (0.67) & \textbf{25.54}$^{*}$ (1.26) \\
  & PSNR fg & 15.58 (0.98) & \underline{16.21} (0.89) & \textbf{17.98}$^{*}$ (1.12) \\
  & PSNR bg & \underline{21.43} (0.66) & 21.18 (0.67) & \textbf{25.59}$^{*}$ (1.27) \\
\hline
\multirow{4}{*}{ADC} 
  & SSIM    & 0.71 (0.04) & 0.71 (0.04) & \textbf{0.79}$^{*}$ (0.03) \\
  & PSNR    & 22.37 (0.52) & \underline{22.47} (0.57) & \textbf{25.42}$^{*}$ (0.60) \\
  & PSNR fg & 15.82 (0.71) & \underline{18.31} (0.82) & \textbf{22.27}$^{*}$ (0.88) \\
  & PSNR bg & 22.42 (0.52) & \underline{22.49} (0.57) & \textbf{25.44}$^{*}$ (0.61) \\
\hline
\end{tabularx}
\end{table}

\subsection{Prostate Lesion Downstream Tasks}
We evaluated the utility of synthetic prostate lesion bpMRIs for three downstream tasks on the test set: csPCa segmentation, patient-level csPCa classification and lesion-level csPCa detection. Evaluation metrics included Dice score for segmentation, area under the receiver operating characteristic curve (AUC) for classification and average precision (AP) for detection. The latter two tasks follow the PI-CAI challenge protocol and were evaluated using the \texttt{picai\_eval} repository~\cite{saha2023artificial}. First, we trained nnU-Net~\cite{isensee2021nnu} with default configuration on the training set. We then trained separate nnU-Net models on training sets augmented with synthetic lesions generated by DiffTumor and PCaPaint at a ratio of three synthetic lesions per real lesion~\cite{chen2025scaling}. Fig.~\ref{fig:qualitative_results} illustrates three non-lesion bpMRI cases, each inpainted with synthetic lesion by both methods, using lesion masks described in Section~\ref{subsec:pcapaint}. Results demonstrate that our method correctly inpaints lesion textures with modality-specific intensity---darker in T2w, brighter in DWI, and darker in ADC---and outperforms DiffTumor. Statistical significance for segmentation was assessed using two-sided Wilcoxon signed-rank test. For classification and detection, significance was determined by bootstrapping AUC and AP differences, CIs excluding zero indicated significance. Table~\ref{tab2} shows that synthetic examples from PCaPaint yield the best performance across all downstream tasks, significantly outperforming DiffTumor. Although the Dice improvement is small, paired differences in the Wilcoxon test consistently favored PCaPaint over DiffTumor: 87 positives, 60 negatives and 25 ties across the 172 test lesion cases.

\begin{figure}[t]
\centering
\includegraphics[width=0.98\textwidth]{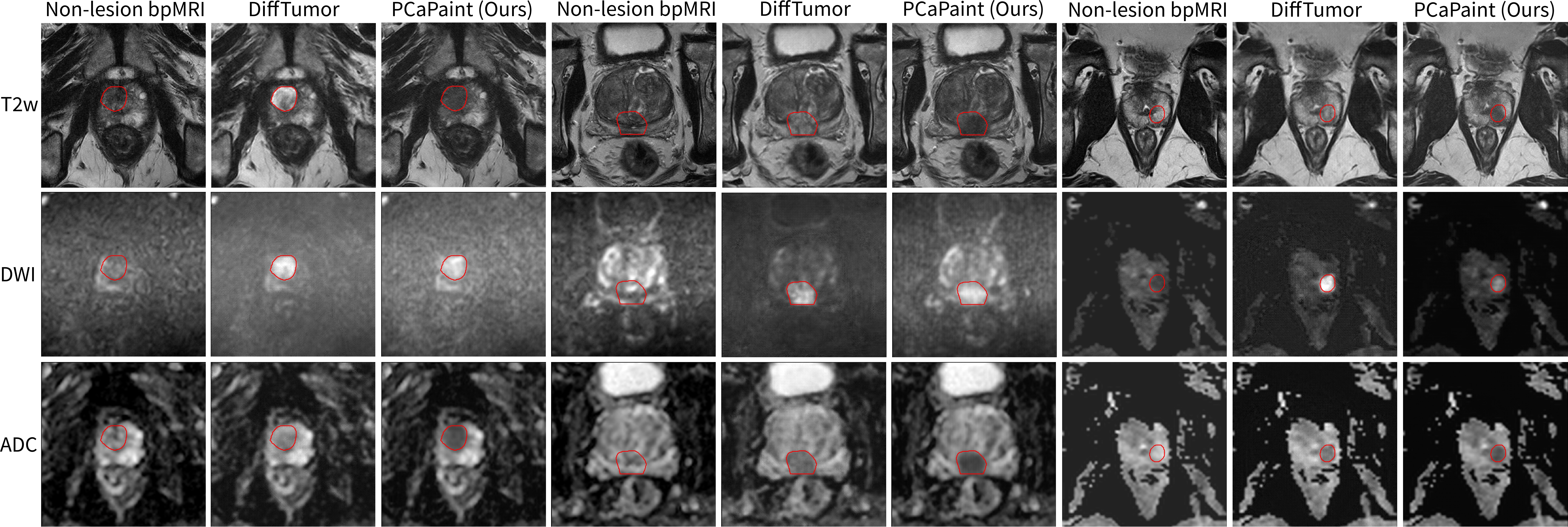}
\caption{Three non-lesion bpMRI examples inpainted with synthetic lesions by DiffTumor~\cite{chen2024towards} and PCaPaint using generated lesion masks.}\label{fig:qualitative_results}
\end{figure}

\begin{table}
\caption{nnU-Net performance on csPCa segmentation, patient-level csPCa classification and lesion-level csPCa detection using real and synthetic data. 95\% bootstrap CI widths are in brackets. Best performance is in \textbf{bold}. $^{*}$ indicates statistically significant differences between Real+DiffTumor and Real+PCaPaint at 5\% level.} \label{tab2}
\begin{tabularx}{\textwidth}{l|X|X|l}
\hline
Downstream Task & Real & Real+DiffTumor & Real+PCaPaint (Ours) \\
\hline
Segmentation, Dice $\uparrow$ & 0.47 (0.08) & 0.47 (0.09) & \textbf{0.48}$^{*}$ (0.08) \\
\hline
Classification, AUC $\uparrow$ & 0.71 (0.09) & 0.71 (0.09) & \textbf{0.76}$^{*}$ (0.08) \\
\hline
Detection, AP $\uparrow$ & 0.28 (0.13) & 0.29 (0.13) & \textbf{0.34}$^{*}$ (0.15) \\
\hline
\end{tabularx}
\end{table}

\section{Conclusion}
In this work, we propose a prostate cancer inpainting method based on LDMs. We address shortcut learning in existing LDM-based tumor inpainting methods by introducing a conditioning strategy with theoretical guarantees. As Proposition~\ref{prop:noise_fill} applies broadly to diffusion-based tumor inpainting, we hypothesize that our conditioning would benefit other imaging modalities and tumor types. Future work will explore these directions. As our synthetic data improves performance on prostate lesion downstream tasks, it will support the development of more robust AI systems for prostate cancer care.

    



%
%

\bibliographystyle{splncs04}
\bibliography{references}

\end{document}